\documentclass[11pt]{article}

\usepackage[margin=1in]{geometry}
\usepackage[utf8]{inputenc}
\usepackage[T1]{fontenc}
\usepackage{times}
\usepackage{microtype}

\usepackage{amsmath,amsfonts,amstext,amssymb,amsbsy,amsopn,amsthm}
\usepackage{mathtools}
\usepackage{mathrsfs}
\usepackage{bm}
\usepackage{dsfont}
\usepackage{esint}
\usepackage{relsize}
\usepackage{nicefrac}

\usepackage{graphicx}
\usepackage{subcaption}
\usepackage{booktabs}
\usepackage{colortbl}
\usepackage{xcolor}
\usepackage{soul}
\usepackage{tikz}
\usepackage[all]{xy}
\usepackage{algorithm}
\usepackage{algorithmic}

\usepackage[numbers,sort&compress]{natbib}
\usepackage{url}
\usepackage[title]{appendix}
\usepackage{xspace}
\usepackage{yfonts}
\usepackage[textsize=tiny]{todonotes}
\usepackage{hyperref}

\newtheorem{theorem}{Theorem}[section]

\newtheorem{lemma}[theorem]{Lemma}

\theoremstyle{definition}
\newtheorem{definition}[theorem]{Definition}
\theoremstyle{remark}

\theoremstyle{remark}

\theoremstyle{remark}

\theoremstyle{remark}

\theoremstyle{remark}

\theoremstyle{remark}

\theoremstyle{remark}

\theoremstyle{remark}

\theoremstyle{definition}

\theoremstyle{definition}
\newtheorem{notation}{Notation}[section]

\theoremstyle{definition}

\theoremstyle{definition}

\title{\vspace{-1.0em}Positional Encoding via Token-Aware Phase Attention}

\author{
Yu (Sid) Wang\thanks{Work done at Meta.}\, \thanks{Correspondence author: \texttt{yuwang2020@gmail.com}}
\quad
Sheng Shen\footnotemark[1]
\quad
R\'emi Munos\thanks{Meta.}
\quad
Hongyuan Zhan\footnotemark[3]
\quad
Yuandong Tian\footnotemark[1]
}

\date{}

\begin{document}

\maketitle

\begin{abstract}
 We prove under practical assumptions that Rotary Positional Embedding (RoPE) introduces an intrinsic distance-dependent bias in attention scores that limits RoPE's ability to model long-context. RoPE extension methods may alleviate this issue, but they typically require post-hoc adjustments after pretraining, such as rescaling or hyperparameters retuning. This paper introduces Token-Aware Phase Attention (TAPA), a new positional encoding method that incorporates a learnable phase function into the attention mechanism. TAPA preserves token interactions over long range, extends to longer contexts with direct and light continual pretraining, extrapolates to unseen lengths, and attains substantially lower perplexity and stronger retrieval performance in the long-context regime than RoPE-style baselines.
\end{abstract}

\section{Introduction}
\label{intro}

Rotary Positional Embedding (RoPE) \citep{su2023roformerenhancedtransformerrotary} is a widely adopted positional encoding method in transformers \citep{Vaswani2017AttentionIA} that applies complex rotations to token representations.

However, RoPE, as originally designed, is not able to extrapolate to context lengths that were not seen during pretraining \citep{Sun2022ALT}, even with extensive continual-pretraining at the extended lengths \citep{Chen2023ExtendingCW,Xiong2023EffectiveLS}. Various extension methods are proposed to improve RoPE's ability to adapt to longer context, such as increasing RoPE base frequency \citep{Xiong2023EffectiveLS}, Position-Interpolation \citep{Chen2023ExtendingCW,Ding2024LongRoPEEL}, YaRN \citep{peng2023yarnefficientcontextwindow} etc.

Many popular publicly available open-source large language models (LLMs) adopt RoPE as their default positional encoding strategy and apply certain RoPE extension methods after pretraining, including LLaMA \citep{touvron2023llama2openfoundation,grattafiori2024llama3herdmodels}, DeepSeek \citep{deepseekai2025deepseekv3technicalreport}, Qwen \citep{yang2024qwen2technicalreport}, Mistral \citep{Jiang2023Mistral7}, Phi \citep{Abdin2024Phi4TR,abdin2024phi3technicalreporthighly}, Kimi \citep{kimiteam2025kimik15scalingreinforcement}, PaLM \citep{Chowdhery2022PaLMSL} etc. 

Despite RoPE’s widespread use in modern LLMs, the reasons behind RoPE's limitations and extensions remain poorly understood. This theory gap motivates our study. In this paper, we prove that RoPE attention carries a non-trivial distance bias—that is, the attention magnitude is dominated by distance between token positions rather than content. This is apparently undesired in language modeling, because relevant information may be downplayed just because it's in a ``bad'' position. In addition, our proof shows that certain RoPE extension methods such as reducing base-frequency and PI indeed mitigate this bias issue.

While RoPE extensions help to some extent, they remain tied to the rotary structure, and rely on manual interventions after pretraining, such as applying an ad hoc formula to rescale input positions, or adjusting base frequency through extensive empirical tuning. The need for such unnatural post-hoc modifications suggests that a more fundamental limitation is present in RoPE’s design, because an ideal positional encoding scheme should be able to fit longer context with minimal long context training and, more importantly, without either hyperparameter or input changes.

We introduce Token-Aware Phase Attention (TAPA), a simple positional encoding framework that inserts a learnable phase function into the attention mechanism. TAPA provably suppresses token-agnostic intrinsic distance bias under mild regularity assumptions and preserves non-degenerate attention variance at arbitrarily long distances. Importantly, it extends a pretrained model to longer contexts via a direct continual-pretraining, without input tweaks or hyperparameter retuning.

Our guiding principle is that positional effects should arise \textbf{only through contextual interactions}. Any distance dependency that persists after averaging over token content is an \emph{intrinsic distance bias} (Definition \ref{IDB}). Such token-agnostic effects can destabilize long-range modeling (Theorem \ref{pre_main_thm} and \ref{main_thm}). TAPA is designed to suppress this undesired bias (Theorem \ref{vanishing_bias_thm}) while preserving positional effects in arbitrarily long range (Theorem \ref{lower_var_thm}).

Empirically, we pretrain a 7B transformer model at 8k, continual-pretrain at 32k, and evaluate up to 64k (Table~\ref{key_results}). TAPA matches baselines through 16k. At 32k, TAPA reaches {11.74} perplexity, reducing perplexity by {$\sim$9.4\%} vs.\ RoPE/PI and {$\sim$3.5\%} vs.\ YaRN. At 64k, TAPA remains {$\sim$11.75} while others blow up to {$\sim$2$\times10^3$}, making it the only method whose test perplexity continues decreasing up to 49K
and remains non-collapsing at 64K.

\section{Theoretical Estimates for Rotary Positional Embedding (RoPE)}

\subsection{Background}
We recall the details of RoPE \citep{su2023roformerenhancedtransformerrotary}, and introduce the notations that are important for our future analysis. 
\begin{notation}[RoPE]
    We let $D$ be transformer head dimension, $1/\theta_0$ be RoPE base frequency, and $\theta_d = \theta_0^{2d/D}$ be the rotation argument of the $d$-th dimension \footnote{Some literature adopt the notation ``$b$'' for base frequency $1/\theta_0$ and refer to $\theta_d$ as the wavelength of RoPE's $d$-th dimension. To simplify notations, we adopt ``$\theta_0$'' and avoid ``$b$''. Increasing RoPE base frequency is simply equivalent to decreasing $\theta_0$.}.
\end{notation}
Denote by $q^{(m)}$ and $k^{(n)}$ the query and key vector representations for tokens at position $m$ and $n$. When no ambiguity is present, we shall drop the upper indices $m$ and $n$ to simplify notations. Denote by $q_{[2d:2d+1]}, k_{[2d:2d+1]}\in \mathbb{R}^2$ the 2-dimensional real vectors that consist of the $(2d)^{\mathrm{th}}$ and $(2d + 1)^{\mathrm{th}}$ coordinates of $q$ and $k$ (for $0\le d\le D/2 - 1$). Further, we complexify both vectors into $q_{[2d:2d+1]}^{\mathbb{C}}$ and $k^{\mathbb{C}}_{[2d:2d+1]}$; that is,
\begin{eqnarray}
\begin{split}
&q_{[2d:2d+1]}^{\mathbb{C}} = q_{2d} + i\cdot q_{2d+1},\\
&k_{[2d:2d+1]}^{\mathbb{C}} = k_{2d} + i\cdot k_{2d+1}.
\end{split}
\end{eqnarray}
The RoPE attention score $\mathrm{Attn}_{\mathrm{RoPE}}(q, k)$ between $q$ at position $m$ and $k$ at position $n$ is defined by
\begin{eqnarray}
\begin{split}
\label{attn_def}
\frac{1}{\sqrt{D}}\text{Re}\big{[}\sum_{d=0}^{D/2 - 1} q_{[2d:2d+1]}^{\mathbb{C}}\cdot(k^{\mathbb{C}}_{[2d:2d+1]})^*\cdot e^{ i(m-n)\theta_d}\big{]},
\end{split}
\end{eqnarray}
where the operation $*$ represents the complex conjugation and $\cdot$ is the multiplication in the complex field. Expanding the right hand side of \eqref{attn_def} and recovering the $m,n$ upper indices, we see that
\begin{eqnarray}
\begin{split}
\label{attn_def_expanded}
\mathrm{Attn}_{\mathrm{RoPE}}(q^{(m)}, k^{(n)}) = \frac{1}{\sqrt{D}}\sum_{d=0}^{D/2-1} \Big{(} &A^{(m,n)}_d\cos2\pi(m - n) \theta_d + B_d^{(m,n)}\sin2\pi(m - n) \theta_d\Big{)}.
\end{split}
\end{eqnarray}
Here we adopt the following handy notations:
\begin{eqnarray}
\begin{split}
\label{real_imag}
A_d^{(m,n)} &\eqqcolon q_{2d}^{(m)}k_{2d}^{(n)} + q_{2d+1}^{(m)}k_{2d+1}^{(n)},\\
B_d^{(m,n)} &\eqqcolon q_{2d}^{(m)}k_{2d+1}^{(n)} - q_{2d+1}^{(m)}k_{2d}^{(n)}.
\end{split}
\end{eqnarray}
Note also we include extra ``$2\pi$''-multiples in \eqref{attn_def_expanded} as they are not essential to RoPE but will greatly simplify expressions in our future analysis.

Lastly we introduce notations necessary for stating the main results in the next section.
\begin{notation}
Let $A_d^{(m,n)}, B_d^{(m,n)}$ be as in \eqref{real_imag}. Define 
\begin{eqnarray}
\begin{split}\label{center_of_mass}
&\mu_{m,n} \eqqcolon \frac{\mathbb{E}_{ q, k} \sum_{d=0}^{D/2-1}A_d^{(m,n)}\cos2\pi(m-n)\theta_d}{\sum_{d=0}^{D/2-1}\cos2\pi(m-n)\theta_d},\\ 
&\nu_{m,n} \eqqcolon \frac{\mathbb{E}_{ q, k} \sum_{d=0}^{D/2-1}B_d^{(m,n)}\sin2\pi(m-n)\theta_d}{\sum_{d=0}^{D/2-1}\sin2\pi(m-n)\theta_d}.
\end{split}
\end{eqnarray}
\end{notation}
Using \eqref{center_of_mass} we further decompose \eqref{attn_def_expanded} as follows:
\begin{eqnarray}
\begin{split}
\label{decomp_attn}
&\frac{1}{\sqrt{D}}\mathrm{Attn}_{\mathrm{RoPE}}(q^{(m)}, k^{(n)}) \\
=& \frac{1}{D}\bigg{(}{\mu_{m,n}} \sum_{d=0}^{D/2 - 1}\cos2\pi(m - n)\theta_d + {\nu_{m,n}} \sum_{d=0}^{D/2 - 1}\sin2\pi(m - n)\theta_d\bigg{)} \\
&+ \frac{1}{D}\bigg{(}\sum_{d=0}^{D/2-1}(A_d - \mu_{m,n})\cos2\pi(m - n)\theta_d+ \sum_{d=0}^{D/2-1}(B_d - \nu_{m,n}) \sin2\pi(m - n)\theta_d\bigg{)}\\
\eqqcolon& \frac{1}{2}\beta^{m,n}_{\mathrm{RoPE}} + Z_{m,n}.
\end{split}
\end{eqnarray}

\subsection{Intrinsic Distance Bias}\label{rope_est}
Noticing that $\mathbb{E}_{q,k}Z_{m,n} = 0$ in \eqref{decomp_attn}, we arrive at a crucial concept of ``\emph{Intrinsic Distance Bias}'':
\begin{definition}
\label{IDB}
    RoPE's \emph{Intrinsic Distance Bias} is given by
\begin{eqnarray}
\begin{split}
\beta^{m,n}_{\mathrm{RoPE}}\eqqcolon \frac{2}{\sqrt{D}}\mathbb{E}_{q^{(m)},k^{(n)}} \mathrm{Attn}_{\mathrm{RoPE}}(q^{(m)}, k^{(n)}).
\end{split}
\end{eqnarray}
\end{definition}
We first clarify the behavior of RoPE across context scales using the main results in this section. Divide the context range into three regimes:
\begin{align*}
\textbf{Local scale}   &:\quad \lesssim \mathcal{O}(\theta_0^{-1/4}), \\
\textbf{Critical scale}&:\quad \sim \mathcal{O}(\theta_0^{-1}), \\
\textbf{Ultra scale}   &:\quad \gg \mathcal{O}(\theta_0^{-1}).
\end{align*}
Theorem \ref{pre_main_thm} shows that RoPE admits unstable attention values at the \textbf{ultra scale}: its \emph{intrinsic distance bias} oscillates heavily and admits two well-separated limit points. Theorem \ref{main_thm} reveals that RoPE favors \textbf{local scale} tokens strictly more than \textbf{critical scale}.

\begin{theorem}[Unstable Long-Context]
\label{pre_main_thm}
If $\{\theta_d\}_{d=1}^{D/2 - 1}$ are $\mathbb{Q}$-linear independent, and it holds that $\mu_{m, n} > c_0, |\mu_{m,n}|, |\nu_{m,n}| < C_0$ for some $c_0, C_0>0$ and  all large $m,n$, then there exists $\gamma^+ \ge c_0$, $\gamma^- \le -c_0$, $\{(m_k^+, n_k^+)\}_k$, and $\{(m_k^-, n_k^-)\}_k$ such that
\begin{eqnarray}
\begin{split}
&\limsup_{k\to\infty}\bigg{|}\beta^{m_k^\pm,n_k^\pm}_{\mathrm{RoPE}}- \gamma^\pm \bigg{|} = \mathcal{O}({1}/{D}).
\end{split}
\end{eqnarray}
\end{theorem}

\paragraph{$\mathbb{Q}$-linear independence}
The condition is satisfied when $\theta_0$ is transcendental, or when $\theta_0^{2/D}$ has algebraic degree higher than $D/2 - 1$.

\paragraph{Interpretation of $\mu_{m,n}>c_0$.}
This condition imposes a non-degenerate relation between the content-dependent inner-product coefficients $A_d^{(m,n)}$ and the RoPE cosine
factors. One typical regime in which it holds is when the coefficients $A_d^{(m,n)}$ have positive mean and controlled variation across dimensions. This is natural because $A_d^{(m,n)}$ captures content-dependent
query--key interactions, whereas $\cos 2\pi(m-n)\theta_d$ arises from
position-dependent rotations; the two terms therefore represent complementary
content and positional components of the attention score. In Appendix \ref{lower_bd_assumption_valid}, we empirically verify these assumptions in a trained 7B RoPE transformer: across 1k sampled position pairs up to 64K distance, $\mu_{m,n}$ remains positive with minimum approximately 1.07, while $|\mu_{m,n}|, |\nu_{m,n}|$ remain uniformly bounded.

We now provide a quantitative characterization of this bias.
\begin{theorem}
\label{main_thm}
Given $c_0, C_0 > 0, \theta_0 < \theta_0(c_0, C_0)$, and $m, n, m^\prime, n^\prime > 0$ with $1 < |m - n| < \theta_0^{-\frac{1}{4}}/8, |m^\prime - n^\prime| > \theta_0^{-1}$. If $\mu_{m, n} > c_0$ and $|\mu_{m,n}|, |\nu_{m,n}|, |\mu_{m^\prime, n^\prime}|, |\nu_{m^\prime,n^\prime}| < C_0$, then
\begin{eqnarray}
\begin{split}
\label{main_gap_estimate}
&\beta^{m,n}_{\mathrm{RoPE}} - \beta^{m^\prime,n^\prime}_{\mathrm{RoPE}}> c_0 / 4
\end{split}
\end{eqnarray}
for all $D > D(\theta_0, |m^\prime - n^\prime|)$.
\end{theorem}

Theorem \ref{main_thm} implies that the effective context length of RoPE does not exceed $\mathcal{O}(\theta_0^{-1})$, since the tokens at or beyond this scale are intrinsically disfavored by the positional bias. This is consistent with the common observation that RoPE attention decreases as relative distance grows \citep{Xiong2023EffectiveLS,Sun2022ALT}.

The attention gap in Eq.~\eqref{main_gap_estimate} makes long-context modeling more challenging, since the model must expend extra learning effort to overcome an intrinsic discrimination against distant tokens to recover relevant long-range information. Notably, the next theorem shows that this gap can be reduced for any fixed pair of positions by decreasing RoPE's $\theta_0$. This aligns with RoPE extension methods based on increasing the base frequency $1/\theta_0$ ~\citep{Xiong2023EffectiveLS} as well as positional interpolation~\citep{Chen2023ExtendingCW}.

\begin{theorem}
\label{freq_red_justify}
    For any $\epsilon > 0$ and $m, n, m^\prime, n^\prime > 0$, the following holds
    \begin{eqnarray}
\begin{split}
\bigg{|}\beta^{m,n}_{\mathrm{RoPE}} - \beta^{m^\prime,n^\prime}_{\mathrm{RoPE}}\bigg{|} 
    < &{|\mu_{m,n} - \mu_{m^\prime, n^\prime}|} + \epsilon
    \end{split}
\end{eqnarray}
    for sufficiently small $\theta_0$ and large $D$.
\end{theorem}
We present the proofs of Theorem \ref{pre_main_thm}, \ref{main_thm}, \ref{freq_red_justify} in Appendix \ref{pre_main_thm_proof}, \ref{rope_proof_mainthm}, and \ref{rope_proof_coro} respectively.

\section{Token-Aware Phase Attention (TAPA)}

We propose \textbf{TAPA}, a positional encoding method that suppresses \emph{Intrinsic Distance Bias} while preserving meaningful interactions at arbitrarily long distances.

\begin{definition}[TAPA]
\label{TAP_def}
Let $q, k$ be representation vectors of query and key located at position $m$ and $n$, $\phi(q, k)$ be any smooth function on the Cartesian space of $(q,k)$. Furthermore, let $\alpha$ be a positive real number, $D$ be the transformer head dimension, and $\mathcal{M}\in \mathbb{R}^{D\times D}$ be any square matrix. Then we define TAPA associated to $(\phi,\mathcal{M},\alpha)$ to be
\begin{eqnarray}
\begin{split}
\label{TAP_attn}
&\mathrm{Attn}_{\mathrm{TAPA}}\Big{(}q, k\Big{)} = q^\top \mathcal{M}k \cdot \cos\Big{(}2\pi |m-n|^\alpha \phi(q, k)\Big{)}.
\end{split}
\end{eqnarray}
\end{definition}
Note TAPA reduces to the standard inner product attention when $\mathcal{M} = I_D$ and $\phi\equiv 0$. Among the many possible choices of $\phi$, this work focuses on the quadratic form\footnote{For convenience, we refer to \eqref{TAP_sta} as a quadratic form, as it is equivalent to the quadratic form defined on Cartesian space of $(q, k)$ given by $\frac{1}{2}\begin{pmatrix}
q^\top& k^\top 
\end{pmatrix}\begin{pmatrix}
0& \mathcal{N} \\
\mathcal{N}^\top& 0
\end{pmatrix}\begin{pmatrix}
q \\
k
\end{pmatrix}$. 
}
\begin{eqnarray}
\begin{split}
\label{TAP_sta}
\phi(q, k) = q^\top \mathcal{N} k, \quad \mathcal{N} \in \mathbb{R}^{D \times D},
\end{split}
\end{eqnarray}
not only for its simplicity and expressivity, but more importantly because it offers the simplest \emph{stationary phase} \citep{Stein1993Harmonic} for suitable choices of $\mathcal{N}$ --- one that possesses a single non-degenerate critical point.

To further simplify TAPA, we segment query and key each into two parts
\begin{eqnarray}
\begin{split}
\label{segment}
q = ({q}_A, {q}_P),\quad k = ({k}_A, {k}_P)
\end{split}
\end{eqnarray}
 such that $q_A, k_A\in\mathbb{R}^{\theta D}$ and $q_P, k_P\in\mathbb{R}^{(1 - \theta) D}$ \footnote{Here the query and key subscripts ``\textbf{A}'' and ``\textbf{P}'' stand for ``\textbf{A}mplitude'' and ``\textbf{P}hase'' respectively.} for some hyperparameter $\theta \in (0, 1)$, and define 
 \begin{eqnarray}
\begin{split}
\label{diag}
&\mathcal{M} = \frac{1}{\sqrt{\theta D}}\cdot\begin{pmatrix}
\mathbf{I}_{\theta D} & \mathbf{0} \\
\mathbf{0} & \mathbf{0}
\end{pmatrix},\quad\mathcal{N} = \frac{1}{\sqrt{(1-\theta)D}}\cdot\begin{pmatrix}
\mathbf{0} & \mathbf{0} \\
\mathbf{0} & \mathbf{I}_{(1 - \theta)D}.
\end{pmatrix}
\end{split}
\end{eqnarray}
 Plugging \eqref{diag} into \eqref{TAP_sta} and then \eqref{TAP_attn}, we obtained the following form of $\mathrm{Attn}_{\mathrm{TAPA}}(q^{(m)}, k^{(n)})$:
\begin{eqnarray}
\begin{split}
\label{diag_attn}
\frac{q_A^\top k_A}{\sqrt{\theta D}} \cdot \cos\Big{(}2\pi |m-n|^\alpha \frac{q_P^\top k_P}{\sqrt{(1-\theta)D}}\Big{)}.
\end{split}
\end{eqnarray}
The hyperparameter $\theta$ controls how parameters are allocated between these two components, and $\alpha$ adjusts the positional sensitivity. Compared with a vanilla transformer, \eqref{diag_attn} uses two disjoint subspaces whose total dimensionality equals the original head dimension; with a fused implementation, this results in a modest constant-factor runtime overhead.

For the rest of the section, we focus on form \eqref{diag_attn} with hyper-parameters $\theta$ and $\alpha$.
\begin{theorem}
\label{vanishing_bias_thm} 
Let $\rho$ be the joint density function of $q_A$, $k_A$, $q_P$, $k_P$, and assume $\rho$ to be Schwartz class. Then there exists $C(\rho, D) > 0$, such that for $m\neq n$ we have
\begin{eqnarray}
\begin{split}
\label{lower_var}
&\Big{|}\mathbb{E}_{q,k}\mathrm{Attn}_{\mathrm{TAPA}}(q^{(m)}, k^{(n)})\Big{|} < C(\rho, D)\cdot {|}{m - n}{|}^{-\alpha(1 - \theta)D}.
\end{split}
\end{eqnarray}
\end{theorem}
\paragraph{Remark}
If $\rho$'s semi-norms up to second degree are bounded by $\mathrm{Poly}(D)$, then $C(\rho, D)$ in \eqref{lower_var} is also $\mathrm{Poly}(D)$. In this case for all $m,n$ with $|m-n| > 1$, their intrinsic distance bias uniformly and exponentially decays in $D$; that is, the undesired intrinsic bias diminishes rapidly as model dimension increases.

In contrast to RoPE's long range instability (e.g. Theorem \ref{pre_main_thm}), Eqn. \eqref{lower_var} tells us that TAPA's \emph{Intrinsic Distance Bias} rapidly decays to $0$ as distance grows. The decay is a result of cancellation in the \emph{oscillatory integral} induced by the quadratic phase. Importantly, \emph{pointwise} attention values need not converge to zero at all. The next Theorem establishes a lower bound on the variance of TAPA, showing that TAPA stays non-degenerate as distance grows, and hence preserves interactions with arbitrarily distant tokens.

\begin{theorem}[Long-Context Non-Degeneracy]\label{lower_var_thm}
    Assume there exists $\sigma_0\neq 0$ such that 
    $$\frac{1}{\theta D}\mathbb{E}_{q_A^{(m)},k_A^{(n)}} |{q_A^{(m)}}^\top k_A^{(n)}|^2\ge \sigma_0^2$$ 
    for all $m,n>0$, then
\begin{eqnarray}
\begin{split}
\label{liminf_lower}
\liminf_{|m - n|\to\infty}\text{Var}_{q,k}(\mathrm{Attn}_{\mathrm{TAPA}}(q^{(m)}, k^{(n)})) \ge \frac{\sigma_0^2}{2}.
\end{split}
\end{eqnarray}
\end{theorem}
The proofs of Theorem \ref{vanishing_bias_thm} and \ref{lower_var_thm} are deferred to Appendix \ref{vanishing_bias_thm_proof} and \ref{lower_var_thm_proof}. 

In Subsection \ref{tap_phase_ablation}, we compare quadratic phases with linear phases. Linear and quadratic phase functions represent the simplest non-stationary and stationary oscillatory families, respectively, while requiring minimal parameterization and computational overhead. From a harmonic analysis perspective, these two classes capture the dominant qualitative behaviors relevant for stability in long-context attention. Higher-order stationary phase functions are expected to exhibit similar asymptotic stability properties with higher computation costs. However, exploring richer phase families is an interesting direction for future work.

Lastly, we design experiments to visualize the distance bias of RoPE and TAPA, and obtain empirical evidence that TAPA is far less affected by distance bias than RoPE. See Appendix \ref{no_attn_bias} for details.

\section{Experiments}\label{experiments}
We pretrain 7B Transformers with the same architecture as LLaMA-3-7B \citep{grattafiori2024llama3herdmodels} from random initialization at an 8k context length, then continual-pretrain at 32k and evaluate up to 64k. Our experiments show that TAPA is able to adapt to 32k by only continual-pretraining on less than $0.25\%$ of pretraining tokens, and extrapolate significantly better to the unseen length of 64k compared to all other baselines. 

Our experiments focus on base language models pretrained from scratch and continually pretrained on longer contexts, without instruction tuning or post-training alignment.

\begin{table*}[t]
\centering
\small
\setlength{\tabcolsep}{3.5pt}
\renewcommand{\arraystretch}{1.05}
\begin{tabular}{l|rrrrrrrr}
\toprule
Context & 1K & 2K & 4K & 8K & 16K & 32K & 49K & 64K \\
\midrule
RoPE (b=5e5) & 12.97 & 12.53 & 12.22 & 11.97 & 11.79 & 12.96 & 938.23 & 2280.16 \\
RoPE (b=2e8) & 13.00 & 12.54 & 12.23 & 11.98 & 11.80 & 12.96 & 942.94 & 2284.72 \\
PI              & 12.99 & 12.54 & 12.23 & 11.98 & 11.80 & 12.97 & 939.17 & 2282.44 \\
YaRN            & 13.05 & 12.60 & 12.29 & 12.03 & 11.85 & 12.16 & 322.14 & 1962.55 \\
\midrule
TAPA            & 13.04 & 12.62 & 12.30 & 12.07 & 11.83 & \textbf{11.74} & \textbf{11.67} & \textbf{11.75} \\
\bottomrule
\end{tabular}
\caption{Test perplexity on PG19 for 7B Transformers pretrained at 8K context and continual-pretrained to 32K, evaluated from 1K to 64K. For RoPE, we include two frequency choices: normal $b=5\times10^5$ and large $b=2\times10^8$.}
\label{key_results}
\end{table*}

\begin{table*}[]
\centering
\normalsize\setlength\tabcolsep{3pt}
\begin{tabular}{l|llllllll}
\toprule
Context window size   & 1024 & 2048 & 4096 & 8192 & 16384  & 32768  \\ \midrule
RoPE (b=5e5) & 13.08 & 12.63 & 12.32 & 12.21 & 5878.17 & 16366.63  \\
YaRN         & 13.09 & 12.64 & 12.33 & 12.22 & 5869.35 & 16342.28  \\
PI           & 13.07 & 12.62 & 12.30 & 12.19 & 5872.29 & 16350.28 \\  \midrule
TAPA         & 13.12 & 12.66 & 12.34 & 12.22 & 17.96 & 122.71  \\ \bottomrule
\end{tabular}
\caption{Test perplexity via directly evaluating 7B transformers (pretrained on 8k context length) on 1k$\sim$32k, without continual-pretrained on 32k.}
\label{key_results_extra}
\end{table*}

\section{Pretraining and long-context continual-pretraining}
\label{pretrain-tot}
\subsection{Pretraining}
\label{app:pretraining}
For fair comparison, we pretrain 7B Transformers from scratch with TAPA \eqref{TAP_def} and RoPE \citep{su2023roformerenhancedtransformerrotary} respectively.

Equation \eqref{diag_attn} involves both amplitude and phase dot products, which falls outside the standard single-dot-product SDPA interface supported by off-the-shelf FlashAttention kernels \citep{Dao2022FlashAttentionFA}. A naïve implementation that materializes both score matrices would incur unnecessary overhead.

To enable a controlled comparison, we implement a fused FlashAttention-style kernel for both TAPA and standard attention using Triton. The kernel streams queries and keys in blocks, avoids materializing $L\times L$ attention matrices, and accumulates amplitude and phase contributions within a single pass prior to the softmax which shares memory reads and reductions.

In this fused setting, TAPA exhibits only a modest constant-factor overhead. Across LLaMA-7B–style configurations and long sequence lengths, we empirically observe a $1.1\times - 1.3\times$ runtime slowdown relative to standard FlashAttention. This is substantially smaller than the $2\times$ slowdown by counting two full-dimensional dot products \footnote{This finding aligns with the theoretical FLOP analysis: the amplitude and phase inner products operate on disjoint subspaces whose total dimensionality equals that of regular attention.}. We release our Triton implementation of this fused kernel to support reproducibility. See Subsection \ref{app:runtime} for runtime experiments details.

For TAPA pretraining, we set $\alpha = 0.1$ and $\theta = 0.5$ in Eq.~\eqref{diag_attn}. For RoPE pretraining, we chose the base frequency $1/\theta_0 = 5 \times 10^5$ following \citet{Xiong2023EffectiveLS}. Due to the high cost of full-scale pretraining, exhaustive hyperparameter ablations for TAPA are infeasible. Nevertheless, we conducted targeted short pretraining runs with early stopping and used early training dynamics to identify stable regions of the hyperparameter space before committing to full-scale runs. For the phase-scaling parameter $\alpha$, we evaluated values in $\{0.01, 0.05, 0.1, 0.2, 0.5, 0.75, 1.0\}$. We observed that $\alpha > 0.5$ occasionally led to unstable training, while $\alpha < 0.05$ resulted in slower convergence and underfitting; $\alpha = 0.1$ consistently provided the most stable behavior and fastest initial convergence. We applied the same procedure to the amplitude--phase split $\theta$ and found $\theta = 0.5$ to yield balanced and stable learning.

The pretraining uses Pile \citep{Gao2020ThePA}, and each training document is chunked into 8k length segments. The pretraining uses 512 $\times$ H100 GPUs with a global batch size of $256$ for a total of 200k steps, which results in a total of 420B tokens. We use AdamW \citep{Loshchilov2017FixingWD} with $\beta_1=0.9,\beta_2=0.95,\epsilon=10^{-8}$ and no weight decay. The optimizer linearly warms up from 0 to the maximum learning rate in 5k steps and then decays according to a cosine schedule to $0.1 \times$ maximum learning rate. We use $10^{-4}$ as the maximum learning rate for RoPE, while for TAPA we find that a smaller learning rate $2\times 10^{-5}$ to be suitable.

\subsection{Long-Context Continual-Pretraining}
\label{app:continual-pretraining}
To extend to long context, we further continual-pretrain the pretrained models with different positional encoding methods on the training split of PG19 \citep{Rae2019CompressiveTF}, where each document is chunked into segments of length 32k. We continual-pretrain with each positional encoding method using a global batch size of $128$ for $500$ steps in total. The optimizer configuration is mostly the same as in pretraining, except that we use $2\times 10^{-5}$ as the maximum learning rate across all methods, and warm up for only $50$ steps.

For RoPE we continual-pretrained with two base frequencies $b=1/\theta_0$. The first reuses the pretraining setting $b=5\times 10^5$, and the second adopts a larger $b=2\times 10^8$ to detect any additional benefit from further increasing $b$.

For TAPA, we keep all hyper-parameters, architectures, and attention computations the same from pretraining. This \textbf{aligns with the key motivation of TAPA's design} --- to enable scaling to longer contexts through direct continual-pretraining and, unlike the RoPE family, does not require any position-scaling or hyperparameter tuning post-pretraining.

For PI \citep{Chen2023ExtendingCW}, we set the max $L^\prime = 65536$ (i.e. 64k), which is the maximal length we will evaluate our models on. For the same reason we set the scale factor to $8$ in YaRN \citep{peng2023yarnefficientcontextwindow}. When continual-pretrained with the increased base frequency approach \citep{Xiong2023EffectiveLS} we experimented with several options ranging from $10^{-6}$ to $2\times 10^{-9}$, and report the best result achieved when base equals to $4\times 10^{-9}$.

\subsection{Long-Context Evaluation}\label{how_to_eval}
We evaluate all long-context-continual-pretrained models on the test split of PG19 \citep{Rae2019CompressiveTF} which consists mostly of long sequence samples. To measure models' performance at different context lengths, we consider segmentation of each document with context window size varying from 1k to 64k in the dyadic fashion. For each context window, we closely follow the sliding window method from \citep{Press2021TrainST} with stride $=256$ to calculate the test loss.

\subsection{Evaluation Results}
We report the test perplexity for multiple positional encoding methods on context window sizes ranging from 1k to 64k on the checkpoints obtained from Subsection \ref{pretrain-tot}.

As shown in Table \ref{key_results}, on short to mid context lengths (1k–16k) all position encodings exhibit a similar, monotonically decreasing perplexity curve (from $\sim$13.0 at 1k to $\sim$11.8 at 16k), with differences within a few hundredths. At 32k, a noticeable difference appears: TAPA attains the lowest perplexity (11.74), followed by YaRN (12.16), while RoPE/PI plateau around 12.96–12.97. Beyond this point, the trends diverge sharply. At 49k–64k, RoPE/PI and YaRN collapse. The perplexities blow up to $\sim$938–2285 (RoPE/PI) and $\sim$322–1963 (YaRN)—whereas TAPA remains stable (11.67 at 49k, 11.75 at 64k). In other words, while YaRN is more resilient than RoPE/PI it still collapses at very long lengths. TAPA is the only method that preserves low perplexity across the entire 1k–64k range, demonstrating substantially stronger long-context robustness and utilization than the alternatives.

In addition, we consider zero-shot long-context perplexity without 32k continual-pretraining. We directly evaluate 7B models pretrained at 8k on 1k--32k and present the results in Table \ref{key_results_extra}. It shows that all position encodings behave similarly on 1k--8k (perplexity decreases from $\sim$13.1 at 1k to $\sim$12.2 at 8k with sub-tenth differences). However, extrapolation beyond 8k fails for {RoPE}/{PI}/{YaRN}: at 16k their perplexities jump to around $ 5.87\times 10^{3}$, and at 32k to roughly $ 1.63\times 10^{4}$. In contrast, {TAPA} degrades gracefully, reaching $17.96$ at 16k and $122.71$ at 32k, which is about $327\times$ and $133\times$ lower than the next-best baselines. These results indicate that without any long-context continual-pretraining, {TAPA} retains substantially better long-range generalization relative to all other baselines.

\begin{table}[t]
\centering
\small
\setlength{\tabcolsep}{2.5pt}
\renewcommand{\arraystretch}{1.05}
\begin{tabular}{l|rrrrrrrr}
\toprule
 & {1k} & {2k} & {4k} & {8k} &
 {16k} & {32k} & {48k} & {64k} \\
\midrule
Linear     & 14.63 & 14.15 & 13.83 & 13.54 & 13.29 & 13.13 & 13.59 & 14.82 \\
Quadratic   & 13.04 & 12.62 & 12.30 & 12.07 & 11.83 & {11.74} & {11.67} & 11.75 \\
\bottomrule
\end{tabular}
\caption{PG19 test perplexity comparison for linear and quadratic phases. 7B Transformer pretrained at 8k, continual-pretrained at 32k, evaluated from 1k--64k.}
\label{key_results_abla}
\end{table}

\subsection{Needle-in-a-Haystack Retrieval under Long-Context Extrapolation}
\label{sec:nih}

To evaluate long-context robustness beyond perplexity, we use Needle-in-a-Haystack (NiH), which targets long-range retrieval behavior of pretrained models.

We perform another continual-pretraining of all checkpoints from Subsection \ref{pretrain-tot} at context length of 32k on additional 5B tokens from open-sourced data. Specifically, we use a mixture of 30\% BookCorpusOpen \citep{Zhu2015AligningBA}, 
35\% Wikipedia \citep{Devlin2019BERTPO}, 
and 35\% C4 \citep{Raffel2019ExploringTL}. This mixture is chosen to complement PG19-only pretraining by introducing diverse long-form narratives, dense factual patterns, web-style noise robustness, and symbol- and number-rich question--answer structures, all of which are important for long-range retrieval. 

\paragraph{Evaluation protocol.}
We follow RULER \citep{hsieh2024rulerwhatsrealcontext} (similar to Table~11 therein) by constructing haystacks from concatenated Paul Graham essays and inserting a single magic-number needle at a random position in each sequence. For each positional encoding method and context length (dyadically ranging from 1k to 64k), we generate 100 independent examples and report the number of successful retrievals. We also report the average accuracy across all evaluated lengths.

\begin{table}[t]
\centering
\small
\setlength{\tabcolsep}{3pt}
\renewcommand{\arraystretch}{1.05}
\begin{tabular}{p{1.2cm}cccccccc}
\toprule
PE Method & 1K & 2K & 4K & 8K & 16K & 32K & 64K & Avg. \\
\midrule
RoPE
& 99 & 100 & 100 & 99 & 96 & 95 & 0 & 84.1 \\
PI
& 100 & 98 & 98 & 99 & 98 & 97 & 0 & 84.3 \\
YaRN
& 100 & 99 & 98 & 98 & 99 & 98 & 0 & 84.6 \\
\midrule
\textbf{TAPA}
& 99 & 98 & 98 & 98 & 99 & \textbf{100} & \textbf{96} & \textbf{98.3} \\
\bottomrule
\end{tabular}
\caption{Needle-in-a-Haystack retrieval accuracy (\%) across context lengths. All models are continually pretrained to 32k context before evaluation. Existing baselines fail to extrapolate to 64k, while TAPA maintains strong retrieval performance.}
\label{tab:nih}
\end{table}

\paragraph{Results.}
As shown in Table~\ref{tab:nih}, all existing baselines completely fail to extrapolate to the unseen length of 64k tokens. In contrast, TAPA maintains high retrieval accuracy at 64k and achieves the best performance at 32k. These results indicate TAPA's strong long-range adaptability and retrieval robustness as well as its competitive short-context performance.

\subsection{Efficiency and Runtime Analysis}
\label{app:runtime}
We measure the runtime overhead of TAPA relative to standard attention using a fused FlashAttention-style Triton kernel on a LLaMA3 7B--style model, across sequence lengths from 2k to 32k. We report per-forward-pass latency averaged over multiple runs. The baseline corresponds to standard FlashAttention, while TAPA uses the same kernel structure with additional fused amplitude and phase accumulation.

As shown in Table~\ref{tab:runtime}, TAPA incurs a modest constant-factor overhead, ranging from $1.28\times$ at 2k to $1.19\times$ at 32k. The relative overhead decreases with sequence length, which indicates that the additional phase computation scales favorably at longer contexts.

Notably, the measured overhead ($1.19\times$--$1.28\times$) is substantially smaller than an intuitive estimate of ``$2\times$'' that counts two inner products per head. This is because our fused kernel reuses the same memory accesses and reductions when accumulating both terms in a single pass over streaming queries and keys.

\begin{table}[t]
\centering
\small
\setlength{\tabcolsep}{6pt}
\renewcommand{\arraystretch}{1.10}
\begin{tabular}{c c c c}
\toprule
Length & Baseline (ms) & TAPA (ms) & TAPA / Baseline \\
\midrule
2048  & 1.089   & 1.398   & 1.283 \\
4096  & 4.138   & 5.233   & 1.265 \\
8192  & 16.434  & 20.537  & 1.250 \\
16384 & 65.119  & 80.480  & 1.236 \\
32768 & 260.545 & 309.010 & 1.186 \\
\bottomrule
\end{tabular}
\caption{Runtime comparison between standard FlashAttention and TAPA using a fused Triton kernel on a LLaMA3 7B--style model across sequence lengths from 2k to 32k. TAPA incurs a modest $1.19\times$--$1.28\times$ overhead that decreases with sequence length.}
\label{tab:runtime}
\end{table}

\subsection{Ablations: TAPA's phase choice}\label{tap_phase_ablation}
We also compare quadratic and linear phases. Quadratic phases perform better across all lengths, while even linear-phase TAPA remains substantially more stable than RoPE-family baselines at long ranges. Full ablations are in Appendix \ref{app:tapa_phase_ablation}.

\section{Related Work}

\paragraph{Positional encoding in Transformers.}
Transformers originally used absolute sinusoidal positional encodings~\citep{Vaswani2017AttentionIA}, later extended to learned embeddings~\citep{Radford2018ImprovingLU,Radford2019LanguageMA,Devlin2019BERTPO}.  
Relative position encodings incorporate distance-dependent signals directly into attention, enabling models to generalize across shifts in position~\citep{Shaw2018SelfAttentionWR,Dai2019TransformerXLAL}.

\paragraph{RoPE and long-context extrapolation.}
Rotary positional embedding (RoPE)~\citep{su2023roformerenhancedtransformerrotary} encodes relative positions via rotations applied to query/key representations, but can become unstable when extrapolating beyond the pretraining context window.  
A large body of follow-up work improves extrapolation primarily through heuristic rescaling of positions or frequencies, including base-frequency scaling~\citep{Xiong2023EffectiveLS,liu2024scalinglawsropebasedextrapolation}, position interpolation (PI)~\citep{Chen2023ExtendingCW}, and non-uniform interpolation schedules such as YaRN~\citep{peng2023yarnefficientcontextwindow} and LongRoPE~\citep{Ding2024LongRoPEEL}.  
While effective in practice, these methods often require case-by-case hyperparameter tuning and do not fully explain the root cause of RoPE extrapolation failure or why certain rescaling rules work.

\paragraph{Beyond RoPE.}
Several non-RoPE methods, including ALiBi \citep{Press2021TrainST}, relative position biases \citep{Shaw2018SelfAttentionWR,Dai2019TransformerXLAL}, NoPE \citep{Haviv2022TransformerLM}, and long-context attention modifications \citep{Wang2024LengthGO,ding2023longnetscalingtransformers1000000000,poli2023hyenahierarchylargerconvolutional}, also target length extrapolation. These methods are important complementary directions. Our experiments focus on RoPE-family baselines because TAPA is designed to address the oscillatory phase mechanism underlying RoPE-style attention, and RoPE-family methods are the dominant positional mechanism in modern LLaMA-style decoder-only LMs. We therefore view RoPE-family methods as the most direct controlled comparison, while a full cross-family comparison is left to future work. A more detailed review of positional encoding methods and long-context extrapolation is provided in Appendix~\ref{app:related_work}.

\bibliographystyle{plainnat}
\bibliography{tapa}

\newpage
\appendix
\onecolumn

\section{Proof of Theorem \ref{pre_main_thm}}\label{pre_main_thm_proof}
\begin{proof}[Proof of Theorem \ref{pre_main_thm}]

Since $\{\theta_d\}_{d=1}^{D/2-1}$ are $\mathbb{Q}$-linearly independent, it follows from Weyl's criterion that $\big{(}(\lambda\theta_1),\cdots, (\lambda\theta_{D/2 - 1})\big{)}$, $\lambda = 1, 2, \cdots$ is uniformly distributed in $[0, 1]^{D/2 - 1}$ (e.g. Theorem 6.3 and Example 6.1 in \cite{kuipers1974uniform}). Here $(r)$ represents the fractional part of a real number $r$. 

For any $\epsilon_k \to 0$, there exist $\lambda_k^\pm\to\infty$ such that
\begin{eqnarray}
\begin{split}
&\bigg{|}\frac{2}{D}  \sum_{d=0}^{D/2 - 1} \sin 2\pi\lambda_{k}^\pm\theta_d \bigg{|} + \bigg{|}\frac{2}{D}  \sum_{d=0}^{D/2 - 1} \cos2\pi\lambda_{k}^\pm\theta_d - (\pm1) \bigg{|} < \epsilon_k + \mathcal{O}(\frac{1}{D}). 
\end{split}
\end{eqnarray}
Let $(m_i^\pm, n_i^\pm)$ satisfy $|m_k^\pm - n_k^\pm| = \lambda_{k}^\pm$. By $C_0$-boundedness of $\mu_{m,n},\nu_{m,n}$ and hence compactness, up to passing to subsequence we have $\mu_{m_k^\pm, n_k^\pm}\to \mu^\pm\in [c_0, C_0]$ and $\nu_{m_k^\pm, n_k^\pm}\to \nu^\pm\in [-C_0, C_0]$. Therefore we have
\begin{eqnarray}
\begin{split}
&\limsup_{i\to\infty}|\beta^{m_i^\pm,n_i^\pm}_{\mathrm{RoPE}} - (\pm\mu^\pm)| \le \mathcal{O}(\frac{1}{D}). 
\end{split}
\end{eqnarray}
Thus we conclude Theorem 2.2 with $\gamma^\pm = \pm\mu^\pm$.
\end{proof}

\section{Proof of Theorem \ref{main_thm}}\label{rope_proof_mainthm}
For brevity, we introduce the following handy notations:
\begin{eqnarray}
\begin{split}
\label{lam_brev}
&\lambda \eqqcolon m - n,\\
&\Lambda \eqqcolon m^\prime - n^\prime.
\end{split}
\end{eqnarray}

The proof of Theorem \ref{main_thm} employs estimates of the following two sums: 
\begin{eqnarray}
\begin{split}
\label{c_s_defs}
\mathcal{C}_D(\lambda)\eqqcolon \frac{1}{D}\sum_{d=0}^{D/2-1} \cos2\pi\lambda \theta_0^{2d/D},\\
\mathcal{S}_D(\lambda)\eqqcolon \frac{1}{D}\sum_{d=0}^{D/2-1} \sin2\pi\lambda \theta_0^{2d/D}.
\end{split}
\end{eqnarray}
\begin{lemma}\label{upper_lemma}
    Given $\theta_0<1/10$, $D > 4|\log \theta_0|$, and $\lambda > 1$, then the following inequalities hold:
\begin{eqnarray}
\begin{split}
\label{c_s_est}
|\mathcal{C}_D(\lambda)| &\le \frac{2}{\theta_0|\log\theta_0|\lambda\pi} + \epsilon(D; \lambda, \theta_0, \alpha),\\
|\mathcal{S}_D(\lambda)| &\le \frac{2}{|\log\theta_0|} + \epsilon(D; \lambda, \theta_0, \alpha),
\end{split}
\end{eqnarray}
for all $\alpha > 0$, where
\begin{eqnarray}
\begin{split}
\label{eps_func}
\epsilon(D; \lambda, \theta_0, \alpha) \eqqcolon \alpha + \frac{4\pi\lambda\theta_0^\alpha}{D}.
\end{split}
\end{eqnarray}
\end{lemma}

\begin{lemma}
\label{lower_lemma}
    Assume $\theta_0<1/10$, $D > 4|\log \theta_0|$, and $\lambda > 1$. If $\lambda\theta_0^{\epsilon_0} < 1/4$ for some $\epsilon_0 > 0$, then we have
\begin{eqnarray}
\begin{split}
\label{c_lower_est}
\mathcal{C}_{D}(\lambda) > \frac{1}{2}\cdot (1 - \epsilon_0)\cdot \cos2\pi\lambda\theta_0^{\epsilon_0} - \frac{1}{|\log\theta_0|} - \epsilon(D; \lambda, \theta_0, \alpha),
\end{split}
\end{eqnarray}
for all $\alpha > 0$ where $\epsilon(D; \lambda, \theta_0, \alpha)$ is defined in \eqref{eps_func}.
\end{lemma}
The proofs of Lemma \ref{upper_lemma} and \ref{lower_lemma} will be given in Appendix \ref{rope_proof_upperlemma} and \ref{rope_proof_lowerlemma}.

\begin{proof}[Proof of Theorem \ref{main_thm}]
Choosing $\epsilon_0=1/4$ and using $\lambda\theta_0^{1/4} < 1/8$, it follows from Lemma \ref{lower_lemma} that
\begin{eqnarray}
\begin{split}
\label{c_lam_lower}
\mathcal{C}_D(\lambda) &> \frac{1}{2}\cdot \frac{3}{4} \cdot\frac{\sqrt{2}}{2} - \frac{2}{|\log\theta_0|} - \epsilon(D; \lambda, \theta_0, \alpha) \\
&> \frac{1}{4} - \frac{2}{|\log\theta_0|} - \epsilon(D; \lambda, \theta_0, \alpha).
\end{split}
\end{eqnarray}
On the other hand, using $\Lambda > \theta_0^{-1}$ and Lemma \ref{upper_lemma}, we have
\begin{eqnarray}
\begin{split}
\label{c_Lam_upper}
\mathcal{C}_D(\Lambda) &< \frac{2}{|\log\theta_0|} + \epsilon(D; \Lambda, \theta_0, \alpha).
\end{split}
\end{eqnarray}
In addition, the following is a direct consequence of Lemma \ref{upper_lemma}:
\begin{eqnarray}
\begin{split}
\label{s_lam_Lam_upper}
\mathcal{S}_D(\lambda) < \frac{2}{|\log\theta_0|} + \epsilon(D; \lambda, \theta_0, \alpha),\\
\mathcal{S}_D(\Lambda) < \frac{2}{|\log\theta_0|} + \epsilon(D; \Lambda, \theta_0, \alpha).
\end{split}
\end{eqnarray}
Combining \eqref{c_lam_lower}, \eqref{c_Lam_upper}, and \eqref{s_lam_Lam_upper}, using the expression \eqref{decomp_attn} and the boundedness conditions in Theorem \ref{main_thm}, we bound the LHS of \eqref{main_gap_estimate} from below as follows:
\begin{eqnarray}
\begin{split}
\label{gap_low_bd}
\beta^{m,n}_{\mathrm{RoPE}} - \beta^{m^\prime,n^\prime}_{\mathrm{RoPE}} &\ge  2\mathcal{C}_D(\lambda)\cdot c_0 - 2|\mathcal{C}_D(\Lambda)\cdot C_0| - 2|\mathcal{S}_D(\lambda)\cdot C_0| -  2|\mathcal{S}_D(\Lambda)\cdot C_0| \\
&> \frac{c_0}{2} - \frac{16C_0}{|\log\theta_0|} - 4C_0\cdot \Big{(}\epsilon(D; \lambda, \theta_0, \alpha) + \epsilon(D; \Lambda, \theta_0, \alpha)\Big{)}.
\end{split}
\end{eqnarray}
By first choosing $\theta_0$ such that $\theta_0 <\theta(c_0, C_0) \eqqcolon \exp(-{128C_0}/{c_0})$, we have
\begin{eqnarray}
\begin{split}
\frac{16C_0}{|\log\theta_0|} < \frac{c_0}{8},
\end{split}
\end{eqnarray}
and then increasing $D$ (beyond some $D(\theta_0, \Lambda)$) such that
\begin{eqnarray}
\begin{split}
4C_0\cdot \Big{(}\epsilon(D; \lambda, \theta_0, \alpha) + \epsilon(D; \Lambda, \theta_0, \alpha)\Big{)} < \frac{c_0}{8}
\end{split}
\end{eqnarray}
for some properly chosen $\alpha$, we arrive at
\begin{eqnarray}
\begin{split}
\label{gap_low_bd_final}
\beta^{m,n}_{\mathrm{RoPE}} - \beta^{m^\prime,n^\prime}_{\mathrm{RoPE}} > \frac{c_0}{2} - \frac{c_0}{8} - \frac{c_0}{8} = \frac{c_0}{4}.
\end{split}
\end{eqnarray}
Thus Theorem \ref{main_thm} is concluded.
\end{proof}

\section{Proof of Theorem \ref{freq_red_justify}}\label{rope_proof_coro}
\begin{proof}
We continue to adopt the the notations in \eqref{lam_brev}. By definition of $\mathcal{C}_{D}(\lambda)$ from \eqref{c_s_defs}, we have the trivial bound 
\begin{eqnarray}
\begin{split}
\label{trivial_upper}
\mathcal{C}_{D}(\lambda) < \frac{1}{2}
\end{split}
\end{eqnarray}
holds for all $\lambda$. Now choose $\theta_0$ to be sufficiently small such that 
\begin{eqnarray}
\begin{split}
\label{lower_conds}
\max(\lambda\theta_0^{\epsilon_0}, \Lambda\theta_0^{\epsilon_0}) < \frac{1}{4},
\end{split}
\end{eqnarray}
for some $\epsilon_0$ to be determined later. Then using \eqref{trivial_upper} and \eqref{lower_conds}, we can apply Lemma \ref{lower_lemma} to see that
\begin{eqnarray}
\begin{split}
\frac{1}{2}(\beta^{m,n}_{\mathrm{RoPE}} - \beta^{m^\prime,n^\prime}_{\mathrm{RoPE}}) &< \frac{|\mu_{m,n}|}{2} - \frac{|\mu_{m^\prime,n^\prime}|}{2}\cdot (1 - \epsilon_0)\cdot \cos2\pi\Lambda\theta_0^{\epsilon_0} - \epsilon(D; \Lambda, \theta_0, \alpha)|\mu_{m^\prime,n^\prime}|,\\
-\frac{1}{2}(\beta^{m,n}_{\mathrm{RoPE}} - \beta^{m^\prime,n^\prime}_{\mathrm{RoPE}}) &< \frac{|\mu_{m^\prime,n^\prime}|}{2} - \frac{|\mu_{m,n}|}{2}\cdot (1 - \epsilon_0)\cdot \cos2\pi\lambda\theta_0^{\epsilon_0} - \epsilon(D; \lambda, \theta_0, \alpha)|\mu_{m,n}|.
\end{split}
\end{eqnarray}
Here for simplicity we adopt the shorthand notation $\mu_{\lambda}$ instead of $\mu_{m,n}$, and same for $\mu_{\Lambda}$. By first choosing $\epsilon_0 < \frac{\epsilon}{8}$, then further decreasing $\theta_0$ such that
\begin{eqnarray}
\begin{split}
|\mu_{m^\prime,n^\prime}|\cos2\pi\Lambda\theta_0^{\epsilon_0} > |\mu_{m^\prime,n^\prime}| - \frac{\epsilon}{8},\\
|\mu_{m,n}|\cos2\pi\lambda\theta_0^{\epsilon_0} > |\mu_{m,n}| - \frac{\epsilon}{8},
\end{split}
\end{eqnarray}
and lastly increasing $D$ such that
\begin{eqnarray}
\begin{split}
\epsilon(D; \lambda, \theta_0, \alpha)|\mu_{m,n}| < \epsilon/8,\\
\epsilon(D; \Lambda, \theta_0, \alpha)|\mu_{m^\prime,n^\prime}| < \epsilon/8,
\end{split}
\end{eqnarray}
we obtain
\begin{eqnarray}
\begin{split}
|\beta^{m,n}_{\mathrm{RoPE}} - \beta^{m^\prime,n^\prime}_{\mathrm{RoPE}}| < \bigg{|}{|\mu_{m,n}| - |\mu_{m^\prime,n^\prime}|}\bigg{|} + \epsilon \le {|\mu_{m,n} - \mu_{m^\prime,n^\prime}}| + \epsilon,
\end{split}
\end{eqnarray}
which proves Theorem \ref{freq_red_justify}.
\end{proof}
\paragraph{Remark} One can directly verify Theorem \ref{freq_red_justify} using the following elementary facts
\begin{eqnarray}
\begin{split}
\lim_{\theta_0\to0}\mathcal{C}_D &= {(\cos2\pi\lambda + D/2 - 1)}/{D},\\
\lim_{\theta_0\to0}\mathcal{S}_D &= {\sin2\pi\lambda}/{D}.
\end{split}
\end{eqnarray}
Namely, choose a sufficiently large $D$ such that $4/D < \epsilon/2$, and then choose $\theta_0$ sufficiently small. But such argument lacks a quantitative understanding of the limiting behavior and the relation among the variables in question. We adopt an alternative proof above using Lemma \ref{lower_lemma} to explicitly quantify the smallness of $\theta_0$ in terms of $\lambda$ and $\epsilon$.

\section{Proof of Lemma \ref{upper_lemma}}\label{rope_proof_upperlemma}

\begin{proof}[Proof of Lemma \ref{upper_lemma}]
Without ambiguity and for simplicity, drop $\lambda$ from the expressions of $\mathcal{C}_D(\lambda), \mathcal{S}_D(\lambda)$ throughout the proof. First let us focus on $\mathcal{C}_D$. By treating $\mathcal{C}_D$ as a Riemann sum we can rewrite it as follows:
\begin{eqnarray}
\begin{split}
\label{decomp_c}
\mathcal{C}_D &= \bigg{(} \frac{1}{D}\sum_{d=0}^{D/2}\cos 2\pi\lambda\theta_0^{2d/D} - \frac{1}{2}\int_{0}^1 \cos 2\pi\lambda \theta_0^xdx\bigg{)} + \frac{1}{2}\int_{0}^1 \cos 2\pi\lambda \theta_0^xdx\\
&= \frac{1}{2}\sum_{d=0}^{D/2} \int_{2d/D}^{2(d+1)/D }\Big{(} \cos2\pi\lambda \theta_0^{2d/D} - \cos2\pi\lambda \theta_0^x\Big{)}dx + \frac{1}{2}\int_{0}^1 \cos 2\pi\lambda \theta_0^xdx\\
&\eqqcolon \frac{1}{2}\Delta + \frac{1}{2} \mathcal{I}.
\end{split}
\end{eqnarray}
We first consider $\Delta\eqqcolon \sum_{d=0}^{D/2}\Delta_d$. For arbitrary $\alpha\in (0, 1)$ we may split the sum into two parts:
\begin{eqnarray}
\begin{split}
\label{two_sums}
\Delta = \sum_{d\le\alpha D/2} \Delta_d + \sum_{d>\alpha D/2}  \Delta_d \eqqcolon \Delta^\prime + \Delta^{\prime\prime}.
\end{split}
\end{eqnarray}
For each term in $\Delta^{\prime}$, we use the fact that cosine functions are bounded by $1$ and control it as follows:
\begin{eqnarray}
\begin{split}
\label{head_sum}
|\Delta_d| \le \frac{2}{D}\cdot 2 = \frac{4}{D}.
\end{split}
\end{eqnarray}
For each term in $\Delta^{\prime\prime}$, we instead use the Lipschitz bound of the integrand:
\begin{eqnarray}
\begin{split}
\label{tail_sum}
|\Delta_d| \le \text{Lip}_d\cdot \frac{4}{D^2},
\end{split}
\end{eqnarray}
where
\begin{eqnarray}
\begin{split}
\label{lip_bound}
\text{Lip}_d \eqqcolon \sup_{[\frac{2d}{D}, \frac{2(d+1)}{D}]}|\frac{d}{dx}\cos 2\pi\lambda \theta_0^x| = 2\pi\lambda |\log \theta_0\cdot \theta_0^x\sin2\pi\lambda\theta_0^x| \le 2\pi\lambda |\log\theta_0|\theta_0^{2d/D}.
\end{split}
\end{eqnarray}
Now plugging both \eqref{head_sum} and \eqref{tail_sum} into \eqref{two_sums}, we obtain:
\begin{eqnarray}
\begin{split}
\label{two_parts_est}
|\Delta^\prime| &\le \frac{D\alpha}{2}\cdot \frac{4}{D} = 2\alpha,\\
|\Delta^{\prime\prime}| &\le \frac{4}{D^2}\cdot 2\pi\lambda |\log\theta_0|\cdot \sum_{d=D\alpha/2 + 1}^{D/2}\theta_0^{2d/D}\\
&\le \frac{4}{D^2}\cdot 2\pi\lambda |\log\theta_0| \cdot \theta_0^{\frac{2}{D}(\frac{D\alpha}{2} + 1)} \cdot \frac{1}{1 - \theta_0^{2/D}}\\
& \le \frac{8\pi\lambda|\log\theta_0|}{D^2}\cdot\theta_0^\alpha \cdot (\frac{1}{1 - \theta_0^{2/D}}  - 1).
\end{split}
\end{eqnarray}
Using 2nd order Taylor's expansion with remainder of Lagrange form, we obtain that
\begin{eqnarray}
\begin{split}
1 - \theta_0^{2/D} \ge |\log\theta_0|\cdot\frac{2}{D} - |\log\theta_0|^2\cdot\frac{4}{D^2}.
\end{split}
\end{eqnarray}
Here we have used the following fact:
\begin{eqnarray}
\begin{split}
\label{summable_est}
\sup_{x\in[0,\frac{2}{D}]}|\frac{d^2}{dx^2}\theta_0^x| \le |\log\theta_0|^2.
\end{split}
\end{eqnarray}
Now inserting \eqref{summable_est} into the estimate of $\Delta^{\prime\prime}$ in \eqref{two_parts_est}, we get
\begin{eqnarray}
\begin{split}
\label{tail_sum_est}
|\Delta^{\prime\prime}|\le \frac{4\pi\lambda}{D}\cdot\theta_0^\alpha\cdot \frac{1}{1 - 2|\log\theta_0| / D} \le \frac{4\pi\lambda}{D}\cdot\theta_0^\alpha \cdot \frac{1}{1 - 1/2} = \frac{8\pi\lambda\theta_0^\alpha}{D}.
\end{split}
\end{eqnarray}
We have now arrived at the bound for the first term in \eqref{two_sums}:
\begin{eqnarray}
\begin{split}
\label{delta_est}
|\frac{\Delta}{2}|\le \alpha + \frac{4\pi\lambda\theta_0^\alpha}{D} \eqqcolon \epsilon(D; \lambda, \theta_0, \alpha),
\end{split}
\end{eqnarray}
where $\epsilon(D; \lambda, \theta_0, \alpha)$ is defined in \eqref{eps_func} and $\alpha\in (0, 1)$ is arbitrary. Next, we estimate the integral term in \eqref{decomp_c}.
By performing a change of variable $y=\theta_0^x$ we see that
\begin{eqnarray}
\begin{split}
\label{osc}
\frac{1}{2}\cdot \mathcal{I} = \frac{1}{2|\log\theta_0|}\int_{\theta_0}^{1} \frac{\cos2\pi \lambda y}{y}dy = \frac{1}{2|\log\theta_0|}\int_{\lambda\theta_0}^{\lambda} \frac{\cos2\pi  y}{y}dy.
\end{split}
\end{eqnarray}
Note the right hand side of \eqref{osc} is an oscillatory integral, so we may employ the cancellation effect to control it. Define
\begin{eqnarray}
\begin{split}
\label{decomp}
n_0 &= \min\{n\in\mathbb{Z}: \lambda\theta_0 \le \frac{1}{4}(4n + 1)\},\\
N_0 &= \max\{n\in\mathbb{Z}: \lambda \ge \frac{1}{4}(4n + 5)\}.
\end{split}
\end{eqnarray}
Decompose the integration interval of \eqref{osc} as follows (for brevity we omit the integrand):
\begin{eqnarray}
\begin{split}
\label{decomp_int}
&\int_{\lambda\theta_0}^{\frac{1}{4}(4n_0 + 1)} + \sum_{n=n_0}^{N_0} \big{(}\int_{\frac{1}{4}(4n + 1)}^{\frac{1}{4}(4n + 3)} +   \int_{\frac{1}{4}(4n + 3)}^{\frac{1}{4}(4n + 5)} \big{)} + \int_{\frac{1}{4}(4N_0 + 5)}^\lambda\\
=& I_*+ \sum_{n=n_0}^{N_0} (I_n^- + I_n^+) + I^*.
\end{split}
\end{eqnarray}
According to \eqref{decomp}, integrals $I_*$ and $I^*$ contain at most a full period of $\cos2\pi y$, and therefore can be trivially bounded:
\begin{eqnarray}
\begin{split}
\label{tail}
|I^*| + |I_*| \le \frac{1}{\lambda \theta_0} + \frac{1}{\lambda - 1} \le \frac{2}{\lambda \theta_0}.
\end{split}
\end{eqnarray}
Here we used the assumption that $\theta_0 < 1/10$ and $\lambda > 1$ in the second inequality above. 
For the sum term in the middle, we have
\begin{eqnarray}
\begin{split}
|I_n^- + I_n^+| = |I_n^-| - |I_n^+| \le \frac{4}{4n+1} \cdot \frac{1}{\pi} - \frac{4}{4n+5} \cdot \frac{1}{\pi} \le \frac{1}{\pi}\cdot \frac{1}{n^2}.
\end{split}
\end{eqnarray}
Here we used the fact that $\cos2\pi y$ is constantly non-positive on the integral range of $I_n^-$, and therefore
\begin{eqnarray}
\begin{split}
|I_n^-| = -\int_{\frac{4n+1}{4}}^{\frac{4n+3}{4}} \frac{\cos2\pi y}{y}dy \le -\frac{4}{4n+1}\int_{\frac{4n+1}{4}}^{\frac{4n+3}{4}} \cos2\pi ydy = \frac{4}{4n+1} \cdot \frac{1}{\pi}.
\end{split}
\end{eqnarray}
Similarly, $\cos2\pi y$ is constantly non-negative on the integral range of $I_n^+$, and thus
\begin{eqnarray}
\begin{split}
-|I_n^+| = -\int_{\frac{4n+3}{4}}^{\frac{4n+5}{4}} \frac{\cos2\pi y}{y}dy \le -\frac{4}{4n+5}\int_{\frac{4n+3}{4}}^{\frac{4n+5}{4}} \cos2\pi ydy = \frac{4}{4n+5} \cdot \frac{1}{\pi}.
\end{split}
\end{eqnarray}

so the sum admits the following bound:
\begin{eqnarray}
\begin{split}
\label{sum_est}
\sum_{n=n_0}^{N_0} (I_n^- + I_n^+)\le \frac{1}{\pi}\sum_{n=n_0}^{N_0}\frac{1}{n^2} \le \frac{1}{\pi}\cdot\frac{1}{n_0 - 1}  \le \frac{1}{\pi}\cdot\frac{2}{\lambda\theta_0},
\end{split}
\end{eqnarray}
where we used the first identity in \eqref{decomp}. Inserting \eqref{tail} and \eqref{sum_est} into \eqref{osc}, we have
\begin{eqnarray}
\begin{split}
\label{osc_est}
\frac{1}{2}\cdot\mathcal{I} \le \frac{2}{\theta_0|\log\theta_0|\lambda\pi}.
\end{split}
\end{eqnarray}
Combining \eqref{delta_est} and \eqref{osc_est}, we concluded the estimate of $\mathcal{C}_D$ in \eqref{c_s_est}. Next we estimate $\mathcal{S}_D$. We point out that most the proofs of bounding $\mathcal{S}_D$ follows the same line as that of $\mathcal{C}_D$, so to avoid repetitive argument, therefore we state without proving all results that are achievable through same techniques as its $\mathcal{C}_D$ counterpart, and only focus on addressing the difference.

First, we conduct a similar decomposition of $\mathcal{S}_D$ as \eqref{decomp_c}, into $\frac{\Delta}{2} + \frac{\mathcal{I}}{2}$. The estimate of $\Delta$ follows from exactly the same lines as that of $\mathcal{C}_D$, hence we omit the details:
\begin{eqnarray}
\begin{split}
\label{delta_s_est}
\frac{\Delta}{2} \le \epsilon(D; \lambda, \theta_0, \alpha).
\end{split}
\end{eqnarray}
To estimate $\frac{\mathcal{I}}{2}$, we again use the change of variable $y=\theta_0^x$ and similar to \eqref{osc} we get
\begin{eqnarray}
\begin{split}
\label{osc_s}
\frac{\mathcal{I}}{2} = \frac{1}{2|\log\theta_0|}\int_{\lambda\theta_0}^{\lambda} \frac{\sin2\pi  y}{y}dy.
\end{split}
\end{eqnarray}
To bound this oscillatory integral we adopt the following decomposition of integral region:
\begin{eqnarray}
\begin{split}
\label{decomp_int_s}
&\int_{\lambda\theta_0}^{1/2} + \sum_{n=0}^{N_0} \big{(}\int_{\frac{1}{2}(2n + 1)}^{\frac{1}{2}(2n + 2)} +   \int_{\frac{1}{2}(2n + 2)}^{\frac{1}{2}(2n + 3)} \big{)} + \int_{\frac{1}{2}(2N_0 + 3)}^\lambda\\
=& I_*+ \sum_{n=0}^{N_0} (I_n^- + I_n^+) + I^*,
\end{split}
\end{eqnarray}
where
\begin{eqnarray}
\begin{split}
\label{decomp_s}
N_0 = \max\{n\in\mathbb{Z}: \lambda \ge \frac{1}{2}(2n + 3)\}.
\end{split}
\end{eqnarray}
Note again that the integrand is non-positive in $I_n^-$, and non-negative in $I_n^+$. Similar to \eqref{tail} we have 
\begin{eqnarray}
\begin{split}
|I_n^- + I_n^+|\le \frac{1}{\pi}\cdot |-\frac{2}{2n+1} + \frac{2}{2n+3}| = \frac{1}{\pi}\cdot\frac{4}{(2n+1)(2n+3)} < \frac{1}{\pi}\cdot\frac{1}{n^2},
\end{split}
\end{eqnarray}
and therefore 
\begin{eqnarray}
\begin{split}
\label{mid_s}
\sum_{n=0}^{N_0} (I_n^- + I_n^+) < \frac{1}{\pi}\sum_{n=0}^{N_0} \frac{1}{n^2} < \frac{\pi}{6} < 1.
\end{split}
\end{eqnarray}
Next, we use the fact that ${\sin2\pi y}/{y}$ is bounded by $1$ on the interval $[0, 1/2]$ to trivially bound $I_*$:
\begin{eqnarray}
\begin{split}
\label{head_s}
|I_*|\le \frac{1}{2}.
\end{split}
\end{eqnarray}
The integral in the last term contains at most a full period, and thus admits the following bound:
\begin{eqnarray}
\begin{split}
\label{tail_s}
|I^*|<\frac{1}{\lambda - 1} \le 1.
\end{split}
\end{eqnarray}
Combining \eqref{mid_s}, \eqref{head_s}, and \eqref{tail_s}, we have
\begin{eqnarray}
\begin{split}
\label{int_s_est}
\frac{\mathcal{I}}{2} < \frac{2}{|\log\theta_0|}.
\end{split}
\end{eqnarray}
Lastly, combining \eqref{delta_s_est} and \eqref{int_s_est}, we proved the estimates of $\mathcal{S}_D$ in \eqref{c_s_est}. Thus we concluded Lemma \ref{upper_lemma}.
\end{proof}

\section{Proof of Lemma \ref{lower_lemma}}\label{rope_proof_lowerlemma}
\begin{proof}[Proof of Lemma \ref{lower_lemma}]
    The proof reuses the decomposition \eqref{decomp_c} and the bound \eqref{delta_est}, but further needs a lower bound for ${\mathcal{I}}/{2}$. 
    First we decompose the right hand side integral of \eqref{osc} as follows:
\begin{eqnarray}
\begin{split}
\frac{1}{2|\log\theta_0|}\bigg{(}\int_{\lambda\theta_0}^{1/4} + \int_{1/4}^{\lambda}\bigg{)}\frac{\cos2\pi y}{y} dy.
\end{split}
\end{eqnarray}
Following the same argument to bound the middle term in \eqref{mid_s}, we have
\begin{eqnarray}
\begin{split}
\label{s_I_lower}
|\frac{1}{2|\log\theta_0|}\int_{1/4}^{\lambda} \frac{\cos2\pi y}{y} dy| < \frac{1}{2|\log\theta_0|} \cdot 2 = \frac{1}{|\log\theta_0|}.
\end{split}
\end{eqnarray}
Next, notice that the integrand stays positive on $[\lambda \theta_0, 1/4]$, we hence have
\begin{eqnarray}
\begin{split}
\label{s_I_upper}
&\frac{1}{2|\log\theta_0|}\int_{\lambda\theta_0}^{1/4} \frac{\cos2\pi y}{y} dy > \frac{1}{2|\log\theta_0|}\int_{\lambda\theta_0}^{\lambda\theta_0^{\epsilon_0}} \frac{\cos2\pi y}{y} dy > \frac{\cos2\pi\lambda \theta_0^{\epsilon_0}}{2|\log\theta_0|}\cdot\int_{\lambda\theta_0}^{\lambda\theta_0^{\epsilon_0}}\frac{dy}{y} \\
=&  \frac{\cos2\pi\lambda \theta_0^{\epsilon_0}}{2|\log\theta_0|} \cdot |\log\theta_0|(1 - \epsilon_0) = \frac{1}{2}\cdot (1 - \epsilon_0)\cdot \cos2\pi\lambda\theta_0^{\epsilon_0}.
\end{split}
\end{eqnarray}
Finally, combining \eqref{delta_est}, \eqref{s_I_lower}, and \eqref{s_I_upper}, we obtain 
\begin{eqnarray}
\begin{split}
|\mathcal{C}_D|&\ge \frac{1}{2}\mathcal{I} - \frac{1}{2}|\Delta| \ge \frac{1}{2|\log\theta_0|}\int_{\lambda\theta_0}^{1/4} \frac{\cos2\pi y}{y} dy - |\frac{1}{2\log\theta_0}\int_{1/4}^{\lambda} \frac{\cos2\pi y}{y} dy| - \frac{1}{2}|\Delta|\\
&\ge (1 - \epsilon_0)\cdot \cos2\pi\lambda\theta_0^{\epsilon_0} - \frac{1}{|\log\theta_0|} - \epsilon(D; \lambda, \theta_0, \alpha),
\end{split}
\end{eqnarray}
which is exactly \eqref{c_lower_est}, and hence proved Lemma \ref{lower_lemma}.
\end{proof}

\section{Proof of Theorem \ref{vanishing_bias_thm}}\label{vanishing_bias_thm_proof}
\begin{proof}[Proof of Theorem \ref{vanishing_bias_thm}]
    For convenience, we introduce the following notations:
\begin{eqnarray}
\begin{split}
x_A \coloneqq (q_A, k_A),\quad x_P \coloneqq (q_P, k_P).
\end{split}
\end{eqnarray}
First let us expand the expression of the expectation of TAPA:
\begin{eqnarray}
\begin{split}
\label{two_ints}
\int_{x_A}\frac{x_A^\top\cdot J_{\theta D} \cdot x_A}{\sqrt{\theta D}}\bigg{(}\int_{x_P} \cos\Big{(} \frac{2\pi|m - n|^\alpha}{\sqrt{(1-\theta)D}}\cdot x_P^\top\cdot J_{(1-\theta)D} \cdot x_P\Big{)}\cdot\rho(x_A, x_P)dx_P\bigg{)} dx_A
\end{split}
\end{eqnarray}
where $J_{d} = \begin{pmatrix}
\mathbf{0} & I_{d} \\
I_{d} & \mathbf{0} 
\end{pmatrix}$. Let us further simplify the expression by writing $\lambda\eqqcolon \frac{2|m - n|^\alpha}{\sqrt{(1-\theta)D}}$. The inner integral of \eqref{two_ints} can now written as
\begin{eqnarray}
\begin{split}
\label{exp_int_control}
\int_{x_P} \cos\Big{(} \pi \lambda\cdot x_P^\top\cdot J \cdot x_P\Big{)}\cdot\rho(x_A, x_P)dx_P = \text{Re}\bigg{(}\int_{x_P} e^{-i\pi \lambda\cdot x_P^\top\cdot J \cdot x_P}\cdot\rho(x_A, x_P)dx_P\bigg{)}.
\end{split}
\end{eqnarray}
where we omitted the subscript $(1-\theta)D$ of $J$ when there is no ambiguity present. Applying Fourier Transform to imaginary Gaussian (e.g. Proposition 6.2 in \citep{wolff2003lectures}), we manage to bound the integral on the right hand side of \eqref{exp_int_control} as follows: 
\begin{eqnarray}
\begin{split}
\label{inner_control}
&|\int_{x_P} e^{-i\pi \lambda\cdot x_P^\top\cdot J \cdot x_P}\cdot\rho(x_A, x_P)dx_P| \\
\le& C\lambda^{-(1-\theta)D}\bigg{(}\sup_{x_P}|\rho(x_A,x_P)| + \lambda^{-2}\sum_{|\alpha_P|= 2 }\sup_{x_P}|D^{\alpha_P}\rho(x_A,x_P)|\bigg{)},
\end{split}
\end{eqnarray}
where the summation is taken over all second order derivatives with respect to the $x_P$ variable, and $C$ is a universal constant. Since $\rho$ is in Schwartz class, the following seminorms of $\rho$ admit fast decay in $x_A$:
\begin{eqnarray}
\begin{split}
\sup_{x_P}|\rho(x_A,x_P)|,\quad \sup_{x_P}|D^{\alpha_P}\rho(x_A,x_P)|.
\end{split}
\end{eqnarray}
Therefore the following function is integrable in $x_A$:
\begin{eqnarray}
\begin{split}
\Phi_\lambda(x_A)\eqqcolon \frac{x_A^\top\cdot J_{\theta D} \cdot x_A}{\sqrt{\theta D}}\cdot \bigg{(}\sup_{x_P}|\rho(x_A,x_P)| + \lambda^{-2}\sum_{|\alpha_P|= 2 }\sup_{x_P}|D^{\alpha_P}\rho(x_A,x_P)|\bigg{)}.
\end{split}
\end{eqnarray}
Note when $|m-n|\neq 0$, by definition we have $\lambda\ge C(D) > 0$. Therefore $\{\Phi_\lambda\}_\lambda$ is uniformly bounded in $L^1$:
\begin{eqnarray}
\begin{split}
\label{expectation_control}
|\int_{x_A}\Phi_\lambda(x_A) dx_A| \le C^\prime(\rho, D).
\end{split}
\end{eqnarray}
Combining \eqref{two_ints}, \eqref{inner_control}, and \eqref{expectation_control}, we have shown that
\begin{eqnarray}
\begin{split}
\label{expectation_control_final}
\Big{|}\mathbb{E}_{q,k}\mathrm{Attn}_{\mathrm{TAPA}}(q^{(m)}, k^{(n)})\Big{|}\le C^\prime(\rho, D)\lambda^{-(1-\theta)D} = C(\rho, D)|m-n|^{-\alpha(1-\theta)D}.
\end{split}
\end{eqnarray}
Thus proves Theorem \ref{vanishing_bias_thm}.
\end{proof}

\section{Proof of Theorem \ref{lower_var_thm}}\label{lower_var_thm_proof}

\begin{proof}
By exploiting the elementary identity $\cos^2x = {(1 + \cos2x)}/{2}$, we may expand the second moment of Attention as follows:
\begin{eqnarray}
\begin{split}
&\mathbb{E}\Big{|}\mathrm{Attn}_{\mathrm{TAPA}} \Big{|}^2 \\
=& \frac{1}{2}\int_{x_A}\frac{(x_A^\top\cdot J_{\theta D} \cdot x_A)^2}{\theta D} dx_A \\
&+ \frac{1}{2} \int_{x_A}\frac{x_A^\top\cdot J_{\theta D} \cdot x_A}{\sqrt{\theta D}}\bigg{(}\int_{x_P} \cos\Big{(} \frac{4\pi|m - n|^\alpha}{\sqrt{(1-\theta)D}}\cdot x_P^\top\cdot J_{(1-\theta)D} \cdot x_P\Big{)}\cdot\rho(x_A, x_P)dx_P\bigg{)} dx_A\\
=& I + II.
\end{split}
\end{eqnarray}
Invoking the assumption, we have
\begin{eqnarray}
\begin{split}
\label{term_i}
I = \frac{1}{2\theta D}\mathbb{E}_{q_A,k_A} |q_A^\top k_A|^2 \ge \frac{\sigma_0^2}{2}
\end{split}
\end{eqnarray}
Next, we follow an exactly identical argument of estimating $\mathbb{E}\mathrm{Attn}_{\theta,\alpha}$ (see Appendix \ref{vanishing_bias_thm_proof}) to obtain
\begin{eqnarray}
\begin{split}
\label{term_ii}
|II| \le C^{\prime\prime}(D,\rho)|m-n|^{-\alpha(1-\theta)D}.
\end{split}
\end{eqnarray}
Lastly, combining \eqref{expectation_control_final}, \eqref{term_i}, \eqref{term_ii}, and taking $\liminf$ as $|m-n|\to\infty$, we proved \eqref{liminf_lower}, which concludes Theorem \ref{lower_var_thm}.
\end{proof}

\section{Validation of the boundedness conditions of Theorem \ref{pre_main_thm} and Theorem \ref{main_thm}}
\label{lower_bd_assumption_valid}

\begin{figure}[h]
\centering
\includegraphics[width=\textwidth, height=5.5cm]{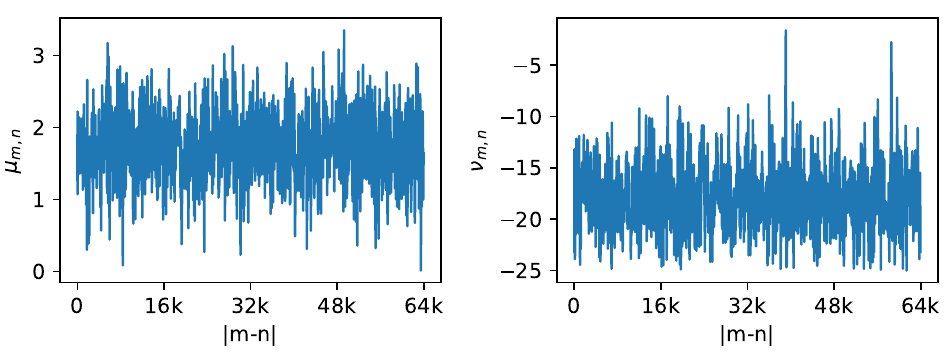}
\caption{Empirical validation of the boundedness conditions in Theorem \ref{pre_main_thm} and Theorem \ref{main_thm}.
    \textbf{Left:} values of $\mu_{m,n}$.
    \textbf{Right:} values of $\nu_{m,n}$.
    All quantities are computed from layer 15 of a 7B RoPE model and remain uniformly bounded across distances up to 64k.}
\label{mu_nu}
\end{figure}

From the proof of Theorem \ref{pre_main_thm} (see Appendix \ref{pre_main_thm_proof}), we see that the condition of Theorem \ref{pre_main_thm} only matters for all sufficiently large positions, so it suffices to verify the conditions of Theorem \ref{main_thm}.

Sample 1000 pairs $(m_i,n_i)$ such that their relative distance $|m_i - n_i|$ is evenly distributed over the interval $[2, 65536]$. For each pair, we compute $\mu_{m,n}, \nu_{m,n}$ values using representations extracted from layer 15 of the 7B model with RoPE with $\theta = 5 \times 10^{-5}$. The left and right panels of Fig.~\ref{mu_nu} show the values of $\mu_{m,n}$ and $\nu_{m,n}$ as functions of $|m-n|$.

Across the full range of relative distances, $\mu_{m,n}$ remains strictly bounded away from zero, with a minimum observed value of approximately $1.07$. In addition, both $|\mu_{m,n}|$ and $|\nu_{m,n}|$ remain uniformly bounded above, with maxima around $3.35$ and $25$ respectively. These empirical observations confirm the boundedness conditions assumed in Theorem \ref{main_thm}.

We remark that this validation is empirical and intended to verify that the required conditions hold for practical model instantiations.

\section{Visualize distance bias}\label{no_attn_bias}
\begin{figure}[h]
\centering
\includegraphics[width=\textwidth, height=5cm]{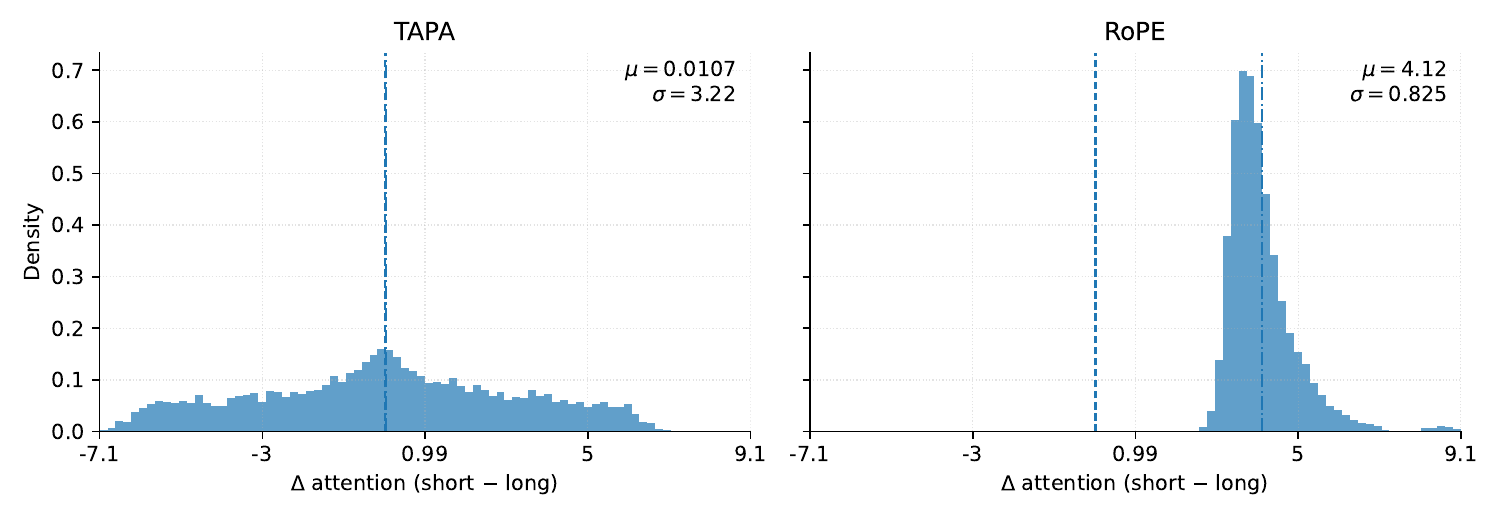}
\caption{\small Empirical distributions of attention score differences, computed over $10000$ randomly sampled pairs with positions drawn from $[0,100]$ (short range) and $[10000,10100]$ (long range). The skewed distribution of RoPE reflects its strong distance bias, whereas the near-symmetric distribution of TAPA indicates that no significant bias is present.}
\label{pos_bias}
\end{figure}
To visualize the distance bias of RoPE and TAPA, we compare the distributions of their attention scores difference between short-range and long-range token pairs. More precisely, given two disjoint intervals $I_{\text{short}}=[0,100]$ and $I_{\text{long}}=[10000,10100]$, we randomly sample $\lambda\in I_{\text{short}}$ and $\Lambda\in I_{\text{long}}$, and compute the difference 
\[
\Delta \;=\; \mathrm{Attn}_{\lambda}-\mathrm{Attn}_{\Lambda}.
\]
We randomly sample $10000$ pairs of attention scores satisfying the above positional constraint from the evaluation results on PG19 test set. The resulting histograms (Figure \ref{pos_bias}) show that TAPA produces a highly symmetric distribution centered near zero ($\mu \approx 0.03, \sigma\approx 3.3$), which indicate low distance bias between near and far token-pairs, while RoPE yields a distribution shifted significantly toward positive values ($\mu \approx 4.1,\ \sigma \approx 0.85$), revealing a systematic bias that favors short-range interactions.

\section{TAPA's phase ablations}
\label{app:tapa_phase_ablation}
We compare two phase functions for TAPA: (i) \emph{quadratic} (stationary) phase in \eqref{TAP_sta} and (ii) \emph{linear} (non-stationary) phase:
\begin{eqnarray}
\begin{split}
\label{TAP_lin}
\phi(q, k) = \frac{1}{\sqrt{(1-\theta)D}}\cdot(q^\top, k^\top)\cdot\mathds{1}_{(1-\theta)D}.
\end{split}
\end{eqnarray}
According to Table \ref{key_results_abla}, TAPA with quadratic phase consistently outperforms the linear variant across all lengths. In the short–mid range (1k–16k), quadratic improves from \(13.04\!\to\!11.83\) versus \(14.63\!\to\!13.29\) for linear—an absolute gap of \(1.3\!-\!1.6\) (\(\approx\!11\%\) relative at 1k and 16k). At longer lengths the divergence grows and stability differs markedly: at 32k, linear reaches \(13.13\) while quadratic is \(11.74\) (\(\Delta=1.39\), \(\approx\!11\%\)); beyond 32k the linear curve becomes non-monotonic and degrades, whereas quadratic remains flat and low. These results align with the intuitions from the theoretical perspective of oscillatory integrals, where non-stationary phases (e.g., linear) induce large, rapidly varying oscillations in representation space, while stationary phases are less sensitive to small representation changes. 

{However, it is worth noting that although TAPA with linear phase is suboptimal compared to the quadratic phase, it still dominates the baselines in Table~\ref{key_results} at long ranges, achieving a significantly lower orders of magnitude: e.g., at 49k/64k it attains \(13.59/14.82\) perplexity versus \(322\!-\!943\) and \(1963\!-\!2285\) for YaRN and RoPE/PI, respectively.}

Overall, the stationary (quadratic) phase yields both {better accuracy} and {greater long-context stability}, while even the linear-phase TAPA retains strong long-context robustness compared to RoPE family.

\section{Extended Related Work}
\label{app:related_work}

\subsection{Positional Encoding in Transformers}
Early Transformers use fixed sinusoidal absolute positional embeddings~\citep{Vaswani2017AttentionIA}, with later work exploring learned absolute embeddings in large-scale language models~\citep{Radford2018ImprovingLU,Radford2019LanguageMA,Devlin2019BERTPO}.  
Absolute methods inject positional signals independent of token content and are typically constrained to a fixed context length.

Relative position encodings address this limitation by allowing attention to depend on pairwise token distances, often implemented as additive biases or learned relative embeddings~\citep{Shaw2018SelfAttentionWR,Dai2019TransformerXLAL,Raffel2019ExploringTL,He2020DeBERTaDB,Press2021TrainST}.  
CAPE~\citep{likhomanenko2021capeencodingrelativepositions} augments sinusoidal embeddings with randomized continuous shifts and scaling during training.

\subsection{RoPE and Extensions for Long-Context Extrapolation}
Rotary positional embedding (RoPE)~\citep{su2023roformerenhancedtransformerrotary} encodes relative positions by applying rotations to query/key representations.  
However, naïve extrapolation of RoPE beyond pretraining context length often leads to degraded perplexity and unstable attention behavior, motivating numerous extensions.

XPos~\citep{Sun2022ALT} introduces an exponential scaling factor on top of RoPE to impose stronger distance-dependent bias.  
Subsequent work observes that XPos still requires additional adjustment, such as base-frequency changes, to maintain performance when extending context length~\citep{Xiong2023EffectiveLS}.  
Related studies systematically analyze the scaling behavior induced by different base frequencies~\citep{liu2024scalinglawsropebasedextrapolation}.

Position interpolation (PI)~\citep{Chen2023ExtendingCW} rescales positions to map longer sequences into the original pretraining range, avoiding direct extrapolation to unseen positions.  
YaRN~\citep{peng2023yarnefficientcontextwindow} and LongRoPE~\citep{Ding2024LongRoPEEL} further propose non-uniform interpolation schedules that modulate scaling strength across RoPE dimensions, aiming to preserve local resolution while enabling longer-range generalization.

Many of these methods rely on hand-designed rescaling rules or post-hoc hyperparameter adjustments, and appropriate settings can be sensitive to model size, training regime, and target context length.  
Some theoretical analyses exist, such as bounding RoPE attention scores~\citep{su2023roformerenhancedtransformerrotary} and studying interpolation error for PI~\citep{Chen2023ExtendingCW}, but they do not directly characterize the failure mode of RoPE extrapolation nor derive principled extrapolation schemes.

\subsection{Non-RoPE Approaches to Extrapolation}
Several approaches aim to extrapolate without relying on RoPE-style modifications.  
NoPE removes explicit positional encodings and relies on the asymmetry of the causal mask to encode positional information implicitly~\citep{Haviv2022TransformerLM,Chi2023LatentPI,Kazemnejad2023TheIO}, though limitations for long-context scaling have been shown both empirically and theoretically~\citep{Wang2024LengthGO,Ma2024MesaExtrapolationAW}.  
LM-Infinite introduces a masking strategy for extrapolation compatible with relative position encodings~\citep{Han2023LMInfiniteZE}.  

Other lines of work modify the attention architecture itself (e.g., Fourier mixing, recurrence, long-convolution, or hybrid designs), including FNet~\citep{Lee2021FNet}, LongNet~\citep{ding2023longnetscalingtransformers1000000000}, RWKV~\citep{peng2023rwkvreinventingrnnstransformer}, and Hyena~\citep{poli2023hyenahierarchylargerconvolutional}.  
Since these methods deviate from standard causal attention, they are complementary to our focus on improving RoPE-style positional encoding under the vanilla Transformer architecture.

\begin{table}[h]
\centering
\small
\begin{tabular}{p{0.22\linewidth} p{0.34\linewidth} p{0.36\linewidth}}
\toprule
Asset & Use & License / Terms \\
\midrule
The Pile~\citep{Gao2020ThePA}
& 8K pretraining corpus
& The EleutherAI Pile code repository is MIT licensed; the Pile is composed of multiple component datasets that may have their own licenses or terms. We cite the dataset and do not redistribute it. \\

PG19~\citep{Rae2019CompressiveTF}
& Long-context continual-pretraining and perplexity evaluation
& The DeepMind PG19 benchmark repository is Apache License 2.0; the benchmark consists of Project Gutenberg books published before 1919. \\

BookCorpusOpen~\citep{Zhu2015AligningBA}
& Part of the Needle-in-a-Haystack continual-pretraining mixture
& No explicit standard license is specified for the BookCorpusOpen dataset card we used. We use it only for training/evaluation and do not redistribute it. \\

Wikipedia~\citep{Devlin2019BERTPO}
& Part of the Needle-in-a-Haystack continual-pretraining mixture
& Wikipedia text is available under Creative Commons Attribution--ShareAlike terms, with some text also available under the GNU Free Documentation License. \\

C4~\citep{Raffel2019ExploringTL}
& Part of the Needle-in-a-Haystack continual-pretraining mixture
& The AllenAI/Hugging Face C4 release is under ODC-BY 1.0; use is also subject to the Common Crawl terms of use for the underlying web content. \\

RULER / Needle-in-a-Haystack protocol~\citep{hsieh2024rulerwhatsrealcontext}
& Long-context retrieval evaluation protocol
& The NVIDIA RULER code repository is Apache License 2.0. We follow the evaluation protocol and generate synthetic needle examples. \\

Paul Graham essays
& Haystack source text for Needle-in-a-Haystack evaluation
& Publicly accessible web essays; no explicit open license was identified. We use them only as evaluation haystack text following the RULER-style setup and do not redistribute the essays. \\

FlashAttention~\citep{Dao2022FlashAttentionFA}
& Runtime comparison and fused-attention implementation reference
& BSD 3-Clause License. \\

Triton
& Custom fused TAPA kernel implementation
& MIT License. \\
\bottomrule
\end{tabular}
\caption{Existing assets used in our experiments and their licenses or terms.}
\label{tab:assets_licenses}
\end{table}

\end{document}